\documentclass[sigconf,nonacm,balance=false]{acmart}
\usepackage{amsmath}
\usepackage{subcaption}
\usepackage{tikz}
\usetikzlibrary{arrows.meta,positioning,calc,fit}
\newcommand{\CCSPrimaryWeight}{500}
\newcommand{\E}{\mathbb{E}}

\setcopyright{none}
\acmDOI{}
\acmISBN{}
\acmYear{2026}
\copyrightyear{2026}
\acmConference{}{}{}
\title{The Scheduler Is a Hidden Retriever: How Asynchrony Shapes Evidence Acquisition in Search Agents}
\makeatletter
\let\CEA@defaultmkauthors\@mkauthors
\renewcommand{\@mkauthors}{%
  \begingroup
  \let\CEA@defaultfnsymbol\@fnsymbol
  \def\@fnsymbol##1{%
    \ifnum##1=2\relax
      \hbox{\fontsize{7}{7}\selectfont\textdagger}%
    \else
      \CEA@defaultfnsymbol{##1}%
    \fi}%
  \CEA@defaultmkauthors
  \endgroup}
\makeatother
\author{Ruiyang Ren}
\authornote{These authors contributed equally.}
\email{reyon\_ren@outlook.com}
\affiliation{%
  \institution{Nanyang Technological University}
  \city{}
  \country{Singapore}}

\author{Yuhao Wang}
\authornotemark[1]
\email{yh.wang500@outlook.com}
\affiliation{%
  \institution{Nanyang Technological University}
  \city{}
  \country{Singapore}}

\author{Yinan Zhang}
\authornote{Corresponding authors.}
\email{yinan.zhang@ntu.edu.sg}
\affiliation{%
  \institution{Nanyang Technological University}
  \city{}
  \country{Singapore}}

\author{Chunyan Miao}
\authornotemark[2]
\email{ascymiao@ntu.edu.sg}
\affiliation{%
  \institution{Nanyang Technological University}
  \city{}
  \country{Singapore}}
\renewcommand{\shortauthors}{Ren et al.}

\renewcommand{\shorttitle}{The Scheduler Is a Hidden Retriever}
\ccsdesc[\CCSPrimaryWeight]{Information systems~Information retrieval}
\keywords{asynchronous search agents, schedule invariance, evidence integration, dependency domains}

\begin{document}
\begin{abstract}
Language-model agents increasingly solve complex tasks through parallel reasoning and asynchronous information gathering. Controlled replay across two model families and three benchmarks shows that delivery timing redirects search decisions made before all results arrive, despite identical initial evidence. Acquisition disagreement reaches up to 24 percentage points above same-order repeat variation. This dependence persists after all initial evidence arrives and through a second search, turning temporary differences in what agents have seen into lasting differences in what they acquire. An acquisition-regret bound with a matching construction shows that invariance at a fixed early arrival count can limit the use of evidence already received. We introduce Confluent Evidence Acquisition (CEA), which instead releases local reasoning when each subquestion's evidence is complete and allocates shared search independently of completion order. For a fixed plan and stable conditional response law, we prove joint invariance of decisions, acquisitions, and answers. With question-only plans across three benchmarks, CEA brings acquisition disagreement to within 1.0 percentage point of repeat variation and achieves accuracy comparable to complete-input controls, improving over early commitment by up to 12.6 percentage points. Replacing CEA's fixed allocation with first-ready allocation under contention adds 21.6 percentage points of acquisition disagreement, while local query-target disagreement remains near repeat variation. Timed comparisons show that CEA reaches decisions earlier and responds faster than the same program waiting for all initial evidence.

\end{abstract}

\maketitle
\section{Introduction}
\label{sec:intro}

Progress in reasoning and tool use enables language-model agents to tackle increasingly complex tasks requiring sustained information gathering across multiple subtasks~\cite{xu2026selfmanager}. Parallel search and asynchronous execution let these subtasks advance while tool results are still arriving~\cite{tan2025ragr1,hooper2026speculative}. In agentic search, reasoning and information acquisition evolve together, as an arriving result can guide a new query before other calls finish. Whether execution timing changes the evidence ultimately acquired is a reliability question for Web agents.

An intermediate fact can arrive in time for final reasoning yet too late to guide search. Figure~\ref{fig:mechanism} shows two recorded runs for a question comparing directors' death dates~\cite{ho2020twowiki}. Both receive the same original documents and have one search opportunity. With a director's name available early, one run retrieves his biography. The other spends its search identifying him. Both eventually receive the name, but only one acquires the date and answers correctly. The scheduler acts as a \emph{hidden retriever} by determining which supplied evidence can guide the next query, a mechanism we call \emph{temporal evidence selection}.

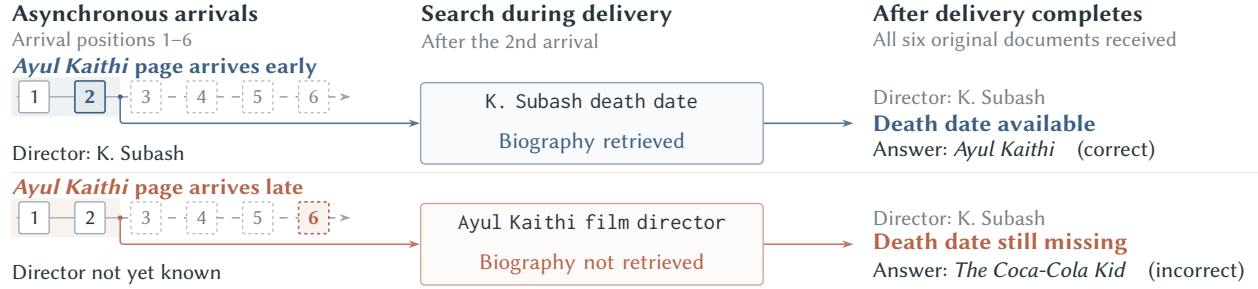
\begin{figure*}[t]
\centering
\definecolor{mechanismBlue}{HTML}{3B5F83}
\definecolor{mechanismOrange}{HTML}{B56349}
\definecolor{mechanismInk}{HTML}{26333C}
\definecolor{mechanismGray}{HTML}{75808A}
\begin{tikzpicture}[x=1cm,y=1cm,
 font=\sffamily\fontsize{8.5}{10}\selectfont,text=mechanismInk,
 heading/.style={font=\sffamily\fontsize{9.1}{10.6}\selectfont\bfseries,
   anchor=west,inner sep=0pt},
 subtitle/.style={font=\sffamily\fontsize{8}{9.2}\selectfont,
   text=mechanismGray,anchor=west,inner sep=0pt},
 detail/.style={align=left,anchor=west,inner sep=0pt},
 document/.style={rounded corners=.8pt,line width=.45pt,draw=mechanismGray!75,
   fill=white,minimum width=.40cm,minimum height=.42cm,inner sep=1pt,
   font=\sffamily\fontsize{8}{9.2}\selectfont},
 pending/.style={draw=mechanismGray!60,text=mechanismGray,
   dash pattern=on 1.2pt off 1.2pt},
 query/.style={rounded corners=1.6pt,line width=.5pt,align=center,
   text width=4.25cm,minimum height=1.02cm,inner sep=4pt},
 flow/.style={-{Stealth[length=1.5mm,width=1.0mm]},line width=.6pt,
   rounded corners=1.2pt}]

\node[anchor=west,font=\sffamily\fontsize{9.1}{10.6}\selectfont]
  at (.03,.87) {\textbf{Question}\quad Which film's director died earlier, \emph{The Coca-Cola Kid} or \emph{Ayul Kaithi}?};
\draw[mechanismGray!32,line width=.4pt] (.16,.48)--(17.20,.48);

\node[heading] at (.16,.12) {Asynchronous arrivals};
\node[heading] at (5.58,.12) {Search during delivery};
\node[heading] at (11.55,.12) {After delivery completes};
\node[subtitle] at (.16,-.22) {Arrival positions 1--6};
\node[subtitle] at (5.58,-.22) {After the 2nd arrival};
\node[subtitle] at (11.55,-.22) {All six original documents received};
\draw[mechanismGray!20,line width=.35pt] (.16,-1.96)--(17.20,-1.96);

\node[detail,font=\sffamily\fontsize{8.5}{10}\selectfont\bfseries,text=mechanismBlue]
  at (.16,-.58) {\emph{Ayul Kaithi} page arrives early};
\node[detail,font=\sffamily\fontsize{8.5}{10}\selectfont\bfseries,text=mechanismOrange]
  at (.16,-2.18) {\emph{Ayul Kaithi} page arrives late};
\fill[mechanismBlue!8,rounded corners=1.2pt] (.16,-1.22) rectangle (1.60,-.70);
\fill[mechanismOrange!8,rounded corners=1.2pt] (.16,-2.82) rectangle (1.60,-2.30);
\draw[mechanismGray!75,line width=.45pt] (.23,-.96)--(1.60,-.96);
\draw[flow,mechanismGray!70,dashed] (1.60,-.96)--(4.65,-.96);
\draw[mechanismGray!75,line width=.45pt] (.23,-2.56)--(1.60,-2.56);
\draw[flow,mechanismGray!70,dashed] (1.60,-2.56)--(4.65,-2.56);
\foreach \arrival/\xpos in {1/.46,2/1.20,3/1.94,4/2.68,5/3.42,6/4.16} {
  \ifnum\arrival=2
    \node[document,draw=mechanismBlue,fill=mechanismBlue!14,line width=.6pt,
      text=mechanismBlue,font=\sffamily\fontsize{8}{9.2}\selectfont\bfseries]
      at (\xpos,-.96) {\arrival};
  \else
    \ifnum\arrival<3
      \node[document] at (\xpos,-.96) {\arrival};
    \else
      \node[document,pending]
        at (\xpos,-.96) {\arrival};
    \fi
  \fi
  \ifnum\arrival=6
    \node[document,pending,draw=mechanismOrange,fill=mechanismOrange!8,line width=.6pt,
      text=mechanismOrange,font=\sffamily\fontsize{8}{9.2}\selectfont\bfseries]
      at (\xpos,-2.56) {\arrival};
  \else
    \ifnum\arrival<3
      \node[document] at (\xpos,-2.56) {\arrival};
    \else
      \node[document,pending]
        at (\xpos,-2.56) {\arrival};
    \fi
  \fi
}
\node[detail,font=\sffamily\fontsize{8.2}{9.6}\selectfont]
  at (.16,-1.70) {Director: K.\ Subash};
\node[detail,font=\sffamily\fontsize{8.2}{9.6}\selectfont]
  at (.16,-3.30) {Director not yet known};

\node[query,draw=mechanismBlue!60,fill=mechanismBlue!3] (queryearly) at (7.84,-1.31)
  {{\ttfamily\fontsize{8.4}{9.8}\selectfont K. Subash death date}\\[5pt]
   \textcolor{mechanismBlue}{Biography retrieved}};
\node[query,draw=mechanismOrange!60,fill=mechanismOrange!3] (querylate) at (7.84,-2.91)
  {{\ttfamily\fontsize{8.4}{9.8}\selectfont Ayul Kaithi film director}\\[5pt]
   \textcolor{mechanismOrange}{Biography not retrieved}};
\fill[mechanismBlue] (1.60,-.96) circle (.032);
\fill[mechanismOrange] (1.60,-2.56) circle (.032);
\draw[flow,mechanismBlue!85] (1.60,-.96)--(1.60,-1.31)--(queryearly.west);
\draw[flow,mechanismOrange!85] (1.60,-2.56)--(1.60,-2.91)--(querylate.west);
\draw[flow,mechanismBlue!85] (queryearly.east)--(11.30,-1.31);
\draw[flow,mechanismOrange!85] (querylate.east)--(11.30,-2.91);

\node[detail,text=mechanismGray,font=\sffamily\fontsize{8.2}{9.6}\selectfont]
  at (11.55,-.96) {Director: K.\ Subash};
\node[detail,font=\sffamily\fontsize{9}{10.5}\selectfont\bfseries,text=mechanismBlue]
  at (11.55,-1.30) {Death date available};
\node[detail] at (11.55,-1.67) {Answer: \emph{Ayul Kaithi}\quad (correct)};
\node[detail,text=mechanismGray,font=\sffamily\fontsize{8.2}{9.6}\selectfont]
  at (11.55,-2.56) {Director: K.\ Subash};
\node[detail,font=\sffamily\fontsize{9}{10.5}\selectfont\bfseries,text=mechanismOrange]
  at (11.55,-2.90) {Death date still missing};
\node[detail] at (11.55,-3.27) {Answer: \emph{The Coca-Cola Kid}\quad (incorrect)};
\end{tikzpicture}
\caption{Arrival timing redirects retrieval. Two recorded runs receive the same six documents in different orders. Each executes one search after the second arrival (shaded prefix). The highlighted film page identifies K.\ Subash, whose biography is absent from the packet. Both runs eventually identify him, but only the upper run acquires his death date.}
\Description{The question asks which of two films, The Coca-Cola Kid or Ayul Kaithi, has the director who died earlier. Two arrival sequences contain the same six original documents. Each sequence shows numbered document arrivals from one to six, with shading over the first two documents available at the search decision and dashed outlines for the later arrivals. The Ayul Kaithi film page is highlighted at arrival two in the upper sequence and arrival six in the lower sequence. In each row, an arrow branches from immediately after the second arrival into a search while the remaining documents are still in flight. In the upper path, the agent already knows the director's name and searches for K. Subash's death date, retrieving his biography. In the lower path, the director is not yet known, so the agent searches for the film's director and misses the biography. After all original documents arrive, both runs identify K. Subash. Only the upper path has his death date and correctly answers Ayul Kaithi; the lower path incorrectly answers The Coca-Cola Kid.}
\label{fig:mechanism}
\end{figure*}

We replay identical evidence packets in different delivery orders to examine \emph{early commitment}, when an agent decides whether and what to search before all results arrive. Across three benchmarks and two model families, early commitment produces acquisition disagreement up to 24 percentage points above same-order repeat variation. Presenting the received documents in a fixed order does not remove the effect, because the evidence available at commitment still differs. Timing also changes whether missing supporting evidence is recovered. Acquisition differences persist after all original evidence arrives and through a second search. In the two-search evaluation, early commitment improves accuracy but leaves acquisition more sensitive to timing. Average accuracy captures task success but cannot establish whether supporting evidence is acquired reproducibly, a distinction that matters when auditing an agent's conclusions~\cite{vashagh2026chronosattack}.

These findings motivate acquisition that is both useful and stable. We formalize \emph{schedule invariance} as an unchanged acquisition or answer distribution across delivery orders. Our acquisition-regret bound, with a matching construction, shows that enforcing invariance at a fixed early arrival count can restrict the use of informative evidence already received. A shared-budget counterexample shows that even invariant local proposals can yield different acquisitions when completion order determines which query executes. Under search contention, this allocation rule adds 21.6 percentage points of acquisition disagreement, while local query-target disagreement remains near repeat variation. Stable acquisition involves both the \emph{evidence dependencies} that govern query formation and the \emph{resource dependencies} that govern execution.

We introduce \emph{Confluent Evidence Acquisition} (CEA) to retain the stability and answer quality of complete-input reasoning while allowing asynchronous progress. Reasoning on each subquestion begins once its evidence is complete, without waiting for unrelated results. A schedule-independent allocation rule determines which proposed search executes. For a fixed plan and stable conditional response law, we prove \emph{epistemic confluence}, meaning that local decisions, executed acquisitions, and the final answer have a schedule-invariant joint distribution. With question-only plans and evidence assigned by retrieval-call ownership, CEA brings acquisition disagreement to within 1.0 percentage point of repeat variation across three benchmarks. Accuracy is comparable to complete-input controls and improves over early commitment by up to 12.6 percentage points. In timed comparisons, CEA reaches decisions earlier and responds faster than the same program waiting for all initial evidence.

We make three contributions.
\begin{itemize}
\item \textbf{Temporal evidence selection.} We isolate how delivery timing redirects retrieval and show why acquisition stability must be assessed alongside evidence quality and answer accuracy.
\item \textbf{Structural limits of early invariance.} We derive an acquisition-regret bound with a matching construction and show how shared-budget allocation introduces a separate source of instability.
\item \textbf{Dependency-scoped coordination.} CEA coordinates local evidence completion and shared search allocation, guaranteeing joint invariance for a fixed plan. It achieves accuracy comparable to complete-input controls with earlier decisions and faster responses.
\end{itemize}

\section{Problem and Evaluation Setup}
\label{sec:setup}

We replay fixed evidence packets over a frozen corpus to isolate delivery-order effects on acquisition and task success from changes in retrieval content or service latency.

\paragraph{Schedule invariance.}
Let $E$ be an initial evidence packet and $\pi$ its delivery order. An agent is \emph{schedule invariant} for a measured outcome if changing $\pi$, with the question, $E$, retriever, and model settings fixed, leaves the outcome distribution unchanged. We distinguish supplied evidence from evidence acquired through search and assess invariance for search decisions, retrieved document sets, and answers.

\paragraph{An evidence acquisition task.}
Answering each question requires evidence from multiple documents. We withhold one supporting document from the supplied packet while keeping it in the retrieval corpus. The six-item packet contains the remaining support and retrieved distractors. A dependent chain retains the bridge document but lacks an endpoint, whereas a comparison lacks support for one side. Recovering the withheld document requires search beyond this packet. Oracle and closed-book reference conditions respectively include the withheld document or supply only the question.

The corpus, excerpts, and content-derived identifiers are fixed before permutation. Deterministic BM25~\cite{robertson2009bm25} returns the same top-five documents for each query. Structural and retrievability checks select questions before inference (Appendix~\ref{app:protocol}).

\paragraph{Datasets and models.}
The primary evaluation uses DeepSeek-V4-Flash on 200 compositional and 200 bridge-comparison questions from 2Wiki~\cite{ho2020twowiki}, 100 HotpotQA bridge questions~\cite{yang2018hotpotqa}, and 100 MuSiQue questions~\cite{trivedi2022musique}. 2Wiki compositional and HotpotQA bridge questions connect evidence along one reasoning chain; 2Wiki bridge-comparison questions compare information from two independent chains. MuSiQue is evaluated without a question-type split. We test cross-model consistency with GPT-5.6 Luna on fixed 100-question subsets of both 2Wiki tasks and the same HotpotQA and MuSiQue cohorts. Supplementary experiments test reasoning controls, two-search acquisition, and budget allocation under low and high contention. Appendix~\ref{app:protocol} gives settings and coverage.

\paragraph{Acquisition from generated plans.}
For 100 fixed questions per benchmark (bridge comparison for 2Wiki), a question-only planner proposes one chain or two independent subquestions. Initial retrieval fills six slots, with each document assigned to the call that retrieved it. All policies use the same plan and packet across schedules. Support is neither inserted nor withheld, and annotations are used only for scoring (Section~\ref{sec:automatic-plans}).

\paragraph{Decision timing and evidence integration.}
The agent chooses \texttt{search(query)} or \texttt{answer(answer)}. We release the decision after two or six arrivals to compare commitment during delivery, as asynchronous runtimes permit~\cite{gim2024asynclm,hooper2026speculative}, with commitment after completion. This isolates evidence availability. Always waiting recovers the complete-input regime. Native strategies update notes in arrival order, whereas batch strategies use canonical identifier order. These yield \texttt{native@2}, \texttt{native@6}, \texttt{wait\_all@2}, and \texttt{wait\_all@6}. Prefix batches fix presentation order but vary which documents are available. Complete batches fix both. CEA reasons on completed domains and coordinates shared search allocation (Section~\ref{sec:method}).

\paragraph{Tracing an acquisition decision.}
We trace one acquisition decision through input completion with a one-search budget. Early trajectories retain their action, feedback, and notes. After six arrivals, a \emph{completed-input probe} records proposals without execution; final answering uses the completed history. The two-search extension allows another search after input completion and any first-search feedback. It measures the union of retrieved documents. Both policies keep the second search available after a provisional answer.

\paragraph{Evaluation measures.}
\emph{Target disagreement} compares query-target categories derived from lexical matches to annotated names and aliases. \emph{Jaccard distance}, $d(A,B)=1-|A\cap B|/|A\cup B|$, compares retrieved document sets; answering gives an empty set, and two empty sets have distance zero. Probe queries use the same retriever, but their results do not enter the history. Answer disagreement compares normalized strings.

We measure search rate, withheld-document recovery over all trajectories and per search, and normalized exact-match accuracy with dataset aliases (Appendix~\ref{app:scoring}). Recovery disagreement records whether runs differ in retrieving the withheld document. We also record readiness position, model calls, and tokens.

\paragraph{Timed execution.}
Paired runs compare local completion with a global barrier using the same program under controlled arrival profiles. Elapsed time runs from dispatch to the final response (Section~\ref{sec:method-evaluation}; Appendix~\ref{app:protocol}).

\paragraph{Estimating schedule effects.}
Permutation experiments evaluate each question under eight fixed delivery orders, with two independent runs per order. Each repeat uses a distinct prompt marker (nonce), shared across orders, so same-order repeats measure run-to-run variation at fixed settings (Appendix~\ref{app:reliability}). We separate schedule effects from repeat variation by subtracting same-order disagreement from cross-order disagreement. For an observed outcome $x_{\pi r}$ under order $\pi$ and repeat $r$, the matched contrast is
\begin{equation}
\begin{split}
\widehat\Delta_{\pi\sigma}
&=\tfrac12\{d(x_{\pi0},x_{\sigma1})+d(x_{\pi1},x_{\sigma0})\}\\
&\quad-\tfrac12\{d(x_{\pi0},x_{\pi1})+d(x_{\sigma0},x_{\sigma1})\}.
\end{split}
\label{eq:centering}
\end{equation}
The mean contrast measures excess disagreement in percentage points (pp). We use question-bootstrap 95\% confidence intervals~\cite{efron1979bootstrap} and pair strategy comparisons within each question. Estimation details appear in Appendices~\ref{app:protocol} and~\ref{app:reliability}.

\section{How Scheduling Shapes Information Acquisition}
\label{sec:measurement}

A search decision determines what additional evidence an agent acquires. We trace how schedule dependence enters acquisition, persists after input completion, and affects evidence quality, then characterize the cost of invariance at a fixed early arrival count.

\begin{figure*}[t]
\centering
\includegraphics[width=\textwidth]{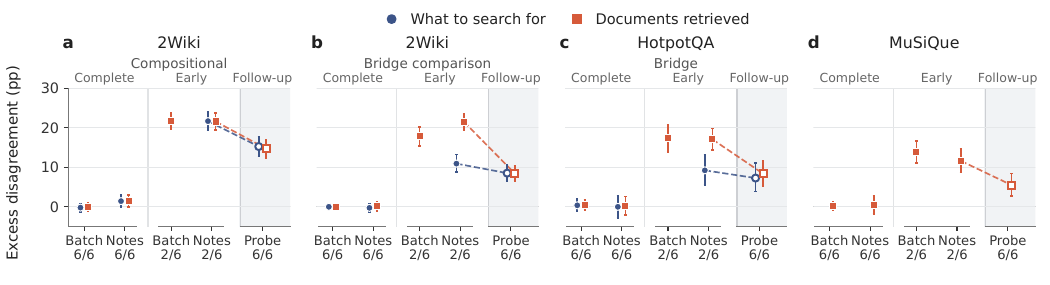}
\caption{Acquisition differences persist after input completion (DeepSeek V4). Batch reads documents in fixed order; Notes updates on arrival. Fractions count received documents. Circles compare what the agent searches for (target categories); squares compare returned documents (Jaccard distance). Dashed lines connect early Notes to its unexecuted probe (shaded, hollow). Zero marks same-order repeat variation; bars show 95\% CIs. Early Batch reports sets only; MuSiQue has no target categories.}
\Description{Four DeepSeek V4 panels in one row show 2Wiki compositional, 2Wiki bridge comparison, HotpotQA bridge, and MuSiQue. Five categorical conditions are grouped into complete-input controls, early decisions, and a follow-up probe. Batch reads received documents together in fixed order; Notes updates working notes in arrival order. Fractions count received original documents. Blue circles compare what the agent searches for, measured by disagreement between target categories. Orange squares compare the returned documents using Jaccard distance; early Batch reports sets only. MuSiQue has no target categories. The four first-decision conditions are independent comparisons. Dashed lines connect only early Notes to its follow-up probe, which retains the early action, notes, and any search feedback after all original documents arrive. Shading and hollow markers distinguish probes, whose queries imply retrieval sets without executing a second search or adding feedback. Zero denotes same-order repeat variation; bars are question-bootstrap 95 percent confidence intervals. Complete-input controls lie near repeat variation; early decisions and their probes retain excess disagreement.}
\label{fig:ladder}
\end{figure*}

\subsection{Scheduling Selects the Evidence for Search}
Delivery can affect a search decision through the evidence available at commitment or the order in which it is processed. Figure~\ref{fig:ladder} shows the primary analysis; Appendix Figure~\ref{fig:ladder-replication} reports the GPT replication. Batch 6/6 (\texttt{wait\_all@6}) presents all six documents in a fixed order, whereas Notes 6/6 (\texttt{native@6}) accumulates notes in arrival order. Both show little excess retrieval disagreement when all six documents are available for the first decision.

The pattern changes when decisions are made after only two arrivals (Notes 2/6, \texttt{native@2}). Retrieved-set disagreement exceeds same-order repeat variation by 11.6--23.8\,pp across all eight model--task settings. Different delivery orders expose the decision to different subsets of the same original packet. Because the retriever is deterministic and no retrieval feedback precedes this decision, the schedule effect enters through whether and what the agent chooses to search.

To distinguish the selected evidence from its presentation, Batch 2/6 (\texttt{wait\_all@2}) presents the first two arrivals together in document-identifier order. Substantial sensitivity remains in every tested setting. On 2Wiki compositional questions, the primary evaluation yields 21.7\,pp of excess disagreement, versus 21.6\,pp with incremental processing. The same pattern appears in the GPT replication, with early incremental/batch contrasts of 12.7/12.5\,pp on HotpotQA and 11.7/11.2\,pp on MuSiQue (Appendix Table~\ref{tab:expansion-details}). Sorting fixes presentation but cannot make evidence still in flight available to the decision.

Scheduling also changes access to useful support. Under early commitment, disagreement about whether the missing document is recovered exceeds the repeat baseline by 7.7\,pp on bridge comparison and 19.3\,pp on compositional questions in the primary evaluation (Appendix~\ref{app:readings}). These controls identify \emph{temporal evidence selection}, in which timing shapes the evidence used to form a query and redirects corpus access despite a fixed packet and retrieval rule.

\subsection{Input Completion Does Not Ensure Informational Convergence}
\label{sec:persistence}
Completing delivery gives agents the same original documents, but their acquired evidence and notes may differ. To test whether these histories still affect information seeking, Probe 6/6 records a new search-or-answer proposal after six arrivals without executing it.

The implied retrieval sets still disagree by 5.5--14.7\,pp above repeat variation across six task--model comparisons (Figure~\ref{fig:ladder}; Appendix Figure~\ref{fig:ladder-replication}). Unlike waiting before the first decision (Notes 6/6), completing delivery after early acquisition leaves subsequent search sensitive to the prior schedule.

Table~\ref{tab:persistence} distinguishes whether to search from what to retrieve. Search demand remains sensitive in every task, reaching 14.7\,pp of excess disagreement on compositional questions. Even when both runs choose to search, their queries imply different document sets, with direction contrasts of 4.3 and 6.2\,pp on HotpotQA and MuSiQue, respectively.
These outcomes separate two consequences of an earlier acquisition history. It can alter whether the agent seeks more information at all, and it can redirect the evidence sought even when both trajectories choose to continue.

\begin{table}[!ht]
\centering\small
\setlength{\tabcolsep}{3pt}
\begin{tabular}{lr@{\hspace{0.35em}}lr@{\hspace{0.35em}}l}
\toprule
Task ($n$) & \multicolumn{2}{c}{Search demand} & \multicolumn{2}{c}{Search direction} \\
\midrule
Comp. (200) & $14.7$ & $[12.4,\,17.1]$ & $2.0$ & $[0.8,\,3.4]$ \\
Bridge (200) & $8.2$ & $[6.2,\,10.3]$ & $1.3$ & $[0.2,\,2.4]$ \\
HotpotQA (100) & $8.6$ & $[5.4,\,11.9]$ & $4.3$ & $[0.6,\,8.1]$ \\
MuSiQue (100) & $3.2$ & $[0.3,\,6.2]$ & $6.2$ & $[2.9,\,9.6]$ \\
\bottomrule
\end{tabular}
\caption{Search demand and direction remain schedule-sensitive after delivery (DeepSeek V4). Excess disagreement above same-order repeats (pp [95\% CI]). Demand compares search/answer choices; direction compares retrieval sets when both runs search. Comp./Bridge denote 2Wiki tasks; eligible pairs: Appendix~\ref{app:protocol}.}
\label{tab:persistence}
\end{table}

\paragraph{Persistence through a second search.}
We allow a second search after input completion and any first-search feedback. On 100 compositional questions, GPT-5.6 yields cumulative acquisition contrasts of 10.4\,pp for early commitment and 2.4\,pp for complete-input commitment, a paired difference of \mbox{8.0\,pp [5.5, 10.4]}. Further search can recover support without eliminating this dependence in the accumulated evidence.

\subsection{Acquisition Quality Beyond Final Accuracy}
\label{sec:cost}

\begin{table*}[t]
\centering\small
\setlength{\tabcolsep}{5pt}
\newlength{\costnumwidth}
\settowidth{\costnumwidth}{$-00.0$}
\newcommand{\costnum}[1]{\makebox[\costnumwidth][r]{$#1$}}
\newcommand{\costhead}[2]{\begin{tabular}[c]{@{}c@{}}#1\\#2\end{tabular}}
\begin{tabular}{lccr@{\hspace{0.35em}}lr@{\hspace{0.35em}}l}
\toprule
& \multicolumn{4}{c}{Incremental reasoning} & \multicolumn{2}{c}{Batch reasoning} \\
\cmidrule(lr){2-5}\cmidrule(l){6-7}
Task & \costhead{Runs with}{search} & \costhead{Searches finding}{support} & \multicolumn{2}{c}{\costhead{Answer}{accuracy}} & \multicolumn{2}{c}{\costhead{Answer}{accuracy}} \\
\midrule
\multicolumn{7}{l}{\textbf{DeepSeek V4}} \\
2Wiki: bridge comparison & \costnum{-10.1} & \costnum{+80.9} & $+30.3$ & $[26.6,\,34.2]$ & $-2.5$ & $[-5.7,\,0.6]$ \\
2Wiki: compositional & \costnum{-13.1} & \costnum{+54.8} & $+17.4$ & $[13.1,\,21.7]$ & $-4.4$ & $[-7.1,\,-1.8]$ \\
HotpotQA & \costnum{-25.1} & \costnum{+35.0} & $+3.6$ & $[-1.1,\,8.3]$ & $-8.9$ & $[-13.1,\,-5.1]$ \\
MuSiQue & \costnum{-24.2} & \costnum{+15.0} & $-1.1$ & $[-3.6,\,1.5]$ & $-5.5$ & $[-8.2,\,-3.0]$ \\
\addlinespace
\multicolumn{7}{l}{\textbf{GPT-5.6}} \\
2Wiki: bridge comparison & \costnum{-14.8} & \costnum{+64.9} & $+36.1$ & $[30.8,\,41.4]$ & $+12.8$ & $[7.4,\,18.2]$ \\
2Wiki: compositional & \costnum{-9.6} & \costnum{+53.6} & $+18.2$ & $[12.5,\,24.1]$ & $+16.6$ & $[10.4,\,22.9]$ \\
HotpotQA & \costnum{-21.1} & \costnum{+30.4} & $+2.1$ & $[-1.4,\,5.8]$ & $-0.6$ & $[-3.5,\,2.1]$ \\
MuSiQue & \costnum{-24.8} & \costnum{+10.7} & $+0.3$ & $[-3.6,\,4.4]$ & $+3.1$ & $[0.8,\,5.7]$ \\
\bottomrule
\end{tabular}
\caption{Waiting makes searches more selective. Differences are complete-input minus early commitment (pp). Support recovery is the fraction of searches finding the withheld document; brackets give paired 95\% CIs for accuracy.}
\label{tab:cost}
\end{table*}

We next examine what these acquisition differences mean for support recovery and answer accuracy. Under the single-search budget, waiting reduces search use and increases recovery per search across all eight incremental settings (Table~\ref{tab:cost}). Both agents also recover more support overall and answer more accurately on both 2Wiki tasks. In the primary bridge-comparison evaluation, recovery rises from 13.2\% to 79.6\%, alongside a 30.3\,pp accuracy gain. Search productivity improves across tasks even though accuracy changes are much smaller on HotpotQA and MuSiQue.

In the two-search evaluation, early commitment reaches 43.6\% accuracy versus 40.9\% for complete-input commitment, a paired gain of 2.7\,pp [0.4, 5.1]. It uses 1.54 versus 1.13 searches per trajectory, with support recovery at 84.4\% and 83.5\%. Higher accuracy and frequent support recovery coexist with greater schedule sensitivity in the acquired evidence (Section~\ref{sec:persistence}).

The complete-input controls clarify what a stable policy should preserve. On the primary bridge-comparison task, batch reasoning almost never searches and reaches 66.4\% accuracy, versus 89.8\% for incremental reasoning. Both have low sensitivity, but incremental reasoning recovers more support (Appendix Table~\ref{tab:cost-full}). Stability is compatible with productive retrieval; the reasoning program shapes how effectively the available evidence guides search.

Different evidence paths can produce the same correct answer (Appendix~\ref{app:input-audit}). The design objective is to preserve useful acquisition while making evidence selection insensitive to delivery timing.

\subsection{The Acquisition Cost of Early Invariance}
\label{sec:early-invariance}
The complete-input controls show that useful acquisition can coexist with low schedule sensitivity. At a fixed early arrival count, however, invariance can restrict how a policy uses evidence it already has. The key is that packets requiring different searches may share a possible prefix.

Fix the question and other decision inputs, and consider one policy across a family of admissible packets. After $m$ of $n$ arrivals, with $1\leq m<n$, it uses only the canonically processed prefix set $S$ to choose a search-or-answer action with law $\mu_S$.

Form a graph $G$ of these prefixes, joining two whenever they can arise from the same packet. Invariance equates decision laws along each edge; a shared prefix carries this constraint across packets. Exact decision invariance is therefore equivalent to a common law within each connected component (Appendix~\ref{app:input-audit}). Let $d_G(S,T)$ denote shortest-path distance.

Let $\ell_E(a)\in[0,1]$ be the loss in acquisition utility relative to the best single action for packet $E$, and $r_E(S)=\E_{a\sim\mu_S}\ell_E(a)$ the expected regret. For $k$ packet-prefix pairs $(E_i,S_i)$ whose prefixes lie in one component $C$, define
\begin{equation}
\begin{aligned}
\Gamma &= \inf_a\frac1k\sum_{i=1}^k\ell_{E_i}(a),\\
\rho_G &= \min_{V\in C}\frac1k\sum_{i=1}^k d_G(V,S_i).
\end{aligned}
\label{eq:acquisition-conflict}
\end{equation}
Here $\Gamma$ is the smallest average utility loss achievable with one action across packets; $\rho_G$ measures how closely their prefixes connect to a common reference.

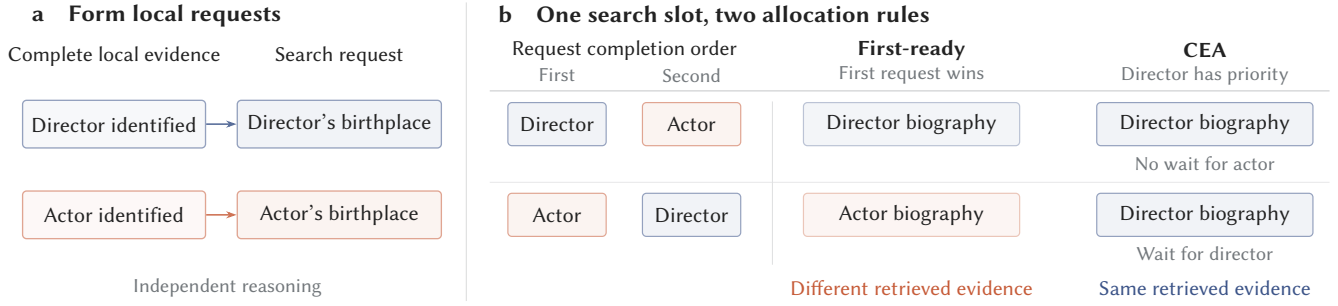
\begin{figure*}[t]
\centering
\definecolor{domainBlue}{HTML}{3C5488}
\definecolor{domainOrange}{HTML}{C65B38}
\definecolor{domainInk}{HTML}{222222}
\definecolor{domainGray}{HTML}{72777B}
\begin{tikzpicture}[x=1cm,y=1cm,
 font=\sffamily\fontsize{8.2}{9.6}\selectfont,text=domainInk,
 heading/.style={anchor=west,font=\sffamily\fontsize{9}{10.5}\selectfont\bfseries},
 label/.style={font=\sffamily\fontsize{8}{9.5}\selectfont},
 detail/.style={text=domainGray,font=\sffamily\fontsize{7.5}{8.8}\selectfont},
 evidence/.style={line width=.5pt,rounded corners=1.5pt,
   minimum width=2.42cm,minimum height=.63cm,inner sep=3pt,align=center},
 request/.style={line width=.5pt,rounded corners=1.5pt,
   minimum width=2.70cm,minimum height=.63cm,inner sep=3pt,align=center},
 ready/.style={line width=.45pt,rounded corners=1.5pt,
   minimum width=1.30cm,minimum height=.57cm,inner sep=2.5pt},
 result/.style={line width=.45pt,rounded corners=1.5pt,
   minimum width=2.86cm,minimum height=.57cm,inner sep=3pt,align=center},
 flow/.style={-{Stealth[length=1.4mm,width=.95mm]},line width=.6pt}]
\node[heading] at (0,.44) {a\quad Form local requests};
\node[label] at (1.20,-.10) {Complete local evidence};
\node[label] at (4.18,-.10) {Search request};
\node[evidence,draw=domainBlue!55,fill=domainBlue!4] (director) at (1.20,-1.00) {Director identified};
\node[request,draw=domainBlue!55,fill=domainBlue!7] (directorquery) at (4.18,-1.00) {Director's birthplace};
\draw[flow,domainBlue!85] (director.east)--(directorquery.west);
\node[evidence,draw=domainOrange!55,fill=domainOrange!4] (actor) at (1.20,-2.20) {Actor identified};
\node[request,draw=domainOrange!55,fill=domainOrange!7] (actorquery) at (4.18,-2.20) {Actor's birthplace};
\draw[flow,domainOrange!85] (actor.east)--(actorquery.west);
\node[detail] at (2.70,-3.16) {Independent reasoning};
\draw[domainGray!25,line width=.4pt] (5.86,.61)--(5.86,-3.34);
\node[heading] at (6.17,.44) {b\quad One search slot, two allocation rules};
\node[label] at (7.95,-.02) {Request completion order};
\node[detail] at (7.06,-.35) {First};
\node[detail] at (8.84,-.35) {Second};
\node[label,font=\sffamily\fontsize{8.2}{9.6}\selectfont\bfseries] at (11.75,-.02) {First-ready};
\node[detail] at (11.75,-.35) {First request wins};
\node[label,font=\sffamily\fontsize{8.2}{9.6}\selectfont\bfseries] at (15.63,-.02) {CEA};
\node[detail] at (15.63,-.35) {Director has priority};
\draw[domainGray!35,line width=.45pt] (6.17,-.56)--(17.28,-.56);
\draw[domainGray!16,line width=.4pt] (6.17,-1.77)--(17.28,-1.77);
\node[ready,draw=domainBlue!55,fill=domainBlue!7] at (7.06,-1.00) {Director};
\node[ready,draw=domainOrange!55,fill=domainOrange!7] at (8.84,-1.00) {Actor};
\node[ready,draw=domainOrange!55,fill=domainOrange!7] at (7.06,-2.20) {Actor};
\node[ready,draw=domainBlue!55,fill=domainBlue!7] at (8.84,-2.20) {Director};
\draw[domainGray!22,line width=.4pt] (9.91,-.70)--(9.91,-2.86);
\node[result,draw=domainBlue!35,fill=domainBlue!6] at (11.75,-1.00) {Director biography};
\node[result,draw=domainOrange!35,fill=domainOrange!6] at (11.75,-2.20) {Actor biography};
\node[result,draw=domainBlue!55,fill=domainBlue!7] at (15.63,-1.00) {Director biography};
\node[detail] at (15.63,-1.51) {No wait for actor};
\node[result,draw=domainBlue!55,fill=domainBlue!7] at (15.63,-2.20) {Director biography};
\node[detail] at (15.63,-2.71) {Wait for director};
\node[text=domainOrange,font=\sffamily\fontsize{8}{9.5}\selectfont] at (11.75,-3.16) {Different retrieved evidence};
\node[text=domainBlue,font=\sffamily\fontsize{8}{9.5}\selectfont] at (15.63,-3.16) {Same retrieved evidence};
\end{tikzpicture}
\caption{Stable requests can still yield different acquisitions. In this deterministic example, evidence identifies a director and an actor but omits their birthplaces. The rows reverse request completion order. With one search slot, first-ready retrieves different biographies; CEA fixes priority before delivery and retrieves the director's biography in both orders.}
\Description{Panel a shows two independent subquestions seeking the birthplaces of a director and an actor. Each forms a fixed search request after receiving its complete local evidence, which identifies the person but omits the birthplace. Panel b compares two alternative request completion orders, shown as rows, under a budget of one search. The first and second columns identify which request finishes first and second; these are not sequential reasoning steps. The remaining columns show the document retrieved by each allocation rule. When the director request finishes first, both rules retrieve the director biography, and CEA searches immediately without waiting for the actor request. When the actor request finishes first, first-ready allocation retrieves the actor biography, whereas CEA waits for the priority director request and retrieves the director biography. The local requests are unchanged, but first-ready retrieves different evidence across the two orders. CEA retrieves the same evidence while preserving independent local reasoning.}
\label{fig:domains}
\end{figure*}

\begin{samepage}
\begin{proposition}[Schedule stability and acquisition regret]
\label{prop:prefix-invariance}
If adjacent prefixes have decision laws at total variation distance at most $\varepsilon$, then
\begin{equation}
\frac1k\sum_{i=1}^k r_{E_i}(S_i)\ge[\Gamma-\rho_G\varepsilon]_+,
\label{eq:stability-regret}
\end{equation}
where $[x]_+=\max\{x,0\}$.
\end{proposition}
\end{samepage}

\begin{proof}
For $V\in C$, shortest paths give $\operatorname{TV}(\mu_{S_i},\mu_V)\le\varepsilon d_G(V,S_i)$. Since $0\le\ell_{E_i}\le1$, bounded-loss comparison~\cite{chen2016bayes} yields
\[
r_{E_i}(S_i)\ge\E_{\mu_V}\ell_{E_i}-\varepsilon d_G(V,S_i).
\]
Every action has average loss at least $\Gamma$, as does a mixture under $\mu_V$. Minimizing average distance over $V$ and using nonnegative regret proves the bound.
\end{proof}

When acquisition needs conflict ($\Gamma>0$), exact invariance forces positive average regret. Allowing $\varepsilon$ variation along each edge relaxes the bound by at most $\rho_G\varepsilon$. Closely connected prefixes leave little room to adapt decisions to packet-specific evidence while preserving stability.

For support recovery, suppose the $k$ packets have distinct, individually recoverable targets, while one search can recover at most $b<k$ of them. Then $\Gamma\ge1-b/k$, so exact invariance leaves an average recovery failure probability of at least $1-b/k$ over the selected pairs.

The matching construction attains this bound on prefixes that identify their packet and its target (Appendix~\ref{app:input-audit}). Each such informative prefix is adjacent to one common, ambiguous prefix, which ties their decision laws together. With $\rho_G=1$, every selected pair attains regret $[1-b/k-\varepsilon]_+$. Because the selected prefixes already identify the targets, this loss comes from the invariance requirement itself.

CEA replaces fixed-count commitment with completion of each subquestion's assigned evidence. The local decision sees the same evidence under every delivery order, while unrelated delivery continues. Shared search budgets introduce a further dependency at execution. CEA also resolves allocation independently of completion order, so scheduling cannot redirect acquisition through competition between stable proposals (Section~\ref{sec:method}).

\section{Confluent Evidence Acquisition}
\label{sec:method}
\label{sec:framework}

We introduce \emph{Confluent Evidence Acquisition} (CEA), which coordinates reasoning and search through their evidence and resource dependencies while allowing independent work to proceed.

\subsection{Evidence and Resource Dependencies}
\label{sec:joint-acquisition}
CEA organizes reasoning around the inputs assigned to each subquestion. A \emph{branch} handles one subquestion, such as establishing a person's birthplace in a comparison. The subquestion and its assigned initial documents form a \emph{dependency domain}, fixed before delivery.

A branch begins reasoning once all documents in its domain have arrived, while other branches may still await their inputs. Changing delivery order changes when this branch is ready, but not which initial evidence it uses. CEA replaces commitment after a fixed arrival count with completion of the decision's inputs (Section~\ref{sec:early-invariance}). Those inputs may still leave facts unresolved, so the branch can request further retrieval.

In an explicit task graph~\cite{xu2023rewoo,kim2024llmcompiler}, retrieval calls identify which evidence belongs to each subquestion. Our controlled experiments use annotated chains, with two domains for comparisons and one for dependent chains. Subquestions provide starting entities and relations without revealing annotated intermediates or answers. With question-only plans, each document is assigned to the subquestion whose retrieval call returned it (Section~\ref{sec:automatic-plans}).

Evidence independence does not imply independent access to search. Even when branches form stable queries, they may compete for the same budget. In Figure~\ref{fig:domains}, giving one search slot to the first completed proposal changes which query executes across schedules. CEA allows independent proposal formation and resolves this competition through an allocation rule that does not depend on completion order.

\subsection{Coordinating Reasoning and Search}
\paragraph{Local reasoning.}
\begin{samepage}
Let $D_1,\ldots,D_k$ be the domains' initial evidence sets, indexed topologically with branch-identifier tie breaking. Readiness occurs at the last assigned document's arrival position
\begin{equation}
r_i(\pi)=\max_{e\in D_i}\operatorname{pos}_{\pi}(e).
\label{eq:readiness}
\end{equation}
\end{samepage}
Incremental CEA processes the ready domain's evidence in canonical order, accumulating private notes before proposing a query or answer. Each request contains only its subquestion, evidence, and notes. Canonical processing fixes the reasoning sequence; private histories keep concurrent branches from altering its context. With one domain, readiness requires all $n=|E|$ items. CEA (Batch) presents the domain's complete evidence in one request.

\paragraph{Search allocation.}
A query proposal does not immediately trigger search. With priorities fixed before delivery, CEA assigns the single search opportunity to the least-indexed requester $j^\star$. Its query can execute once its proposal is available and every higher-priority domain has declined search. Lower-priority branches need not finish their proposals. If all domains decline, no retrieval occurs.

More generally, allocation can be any fixed function of domain-indexed proposals and schedule-independent inputs. Search is eligible once pending proposals cannot change the allocation. Priorities may follow the plan or a question-seeded hash of domain identifiers.

\paragraph{Feedback and synthesis.}
Search results return only to the selected branch, which replaces its proposal with a local answer without altering other branches' query inputs. Final synthesis orders all local outputs by domain, retaining subquestion and root-entity labels. Completion times cannot change their presentation.

\subsection{Joint Invariance and Search Release}
\label{sec:guarantee}

We call joint schedule invariance of local decisions, executed acquisitions, and the final answer \emph{epistemic confluence}. Let $X$ be a complete request, $Y$ its response, and $H$ the preceding execution history. Assumption (A-K) states that $\mathcal L(Y\mid X,H)=K(\cdot\mid X)$ at fixed model settings, so the response distribution depends only on the request. With deterministic retrieval, these conditions yield the following guarantee.

\begin{samepage}
\begin{proposition}[Epistemic confluence of CEA]
\label{prop:domain-invariance}
Under these conditions, CEA preserves the joint law of domain-indexed decisions, executed acquisitions, and the final answer across permutations of the original packet. Readiness and tool-release positions may change.
\end{proposition}
\end{samepage}

To see why, couple corresponding logical calls with seeds shared across schedules and independent across calls. Each branch processes identical domain inputs in the same order, so induction gives matching private histories and proposals. Stable allocation selects the same requester. Deterministic retrieval, private feedback, and canonical synthesis then preserve agreement through the final answer. Conditioning on a plan generated independently of delivery and averaging over plans preserves the guarantee.

The same dependencies determine how early search can start. Let $\tau_i$ be the time at which domain $i$'s proposal becomes available. For the realized winner $j^\star$ under fixed priority, search requires its proposal and the higher-priority declines, giving the earliest eligible release
\begin{equation}
\tau_{\mathrm{eligible}}=\max_{i\le j^\star}\tau_i.
\label{eq:allocation-release}
\end{equation}
Local reasoning can overlap delivery to other domains; once allocation is settled, search can overlap their unfinished reasoning. Final synthesis waits for all domain outputs and any selected retrieval feedback. Appendix~\ref{app:input-audit} proves the release criterion and reports input checks; Section~\ref{sec:method-evaluation} tests acquisition quality and asynchronous progress.

\begin{figure*}[t]
\centering
\begin{subfigure}[t]{.325\textwidth}
\centering
\includegraphics[width=\linewidth]{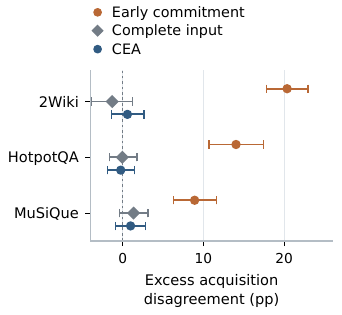}
\caption{Acquisition stability}
\end{subfigure}\hfill
\begin{subfigure}[t]{.345\textwidth}
\centering
\includegraphics[width=\linewidth]{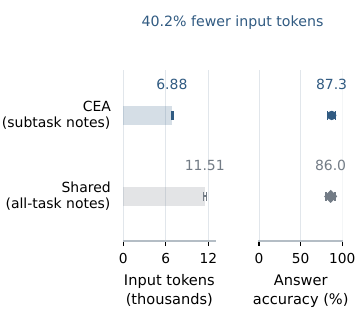}
\caption{Reasoning cost and accuracy}
\end{subfigure}\hfill
\begin{subfigure}[t]{.300\textwidth}
\centering
\includegraphics[width=\linewidth]{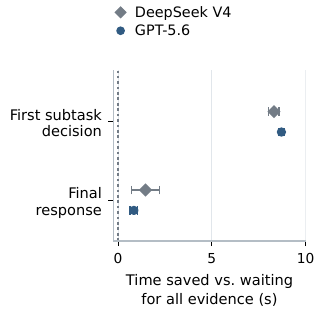}
\caption{Time saved by local execution}
\end{subfigure}
\caption{CEA evaluation. Panels (a,b) use GPT-5.6. (a) Acquisition disagreement with question-only plans under early commitment (2/6 arrivals), complete-input shared history (6/6), and CEA. Zero marks same-order repeat variation. (b) Mean input-token cost per run and accuracy with matched reasoning calls. CEA keeps subtask notes separate; Shared combines them. (c) Time saved against the same program waiting for all initial evidence (20 DeepSeek V4 and 100 GPT-5.6 questions). Whiskers show 95\% CIs, paired in (c).}
\Description{Three panels evaluate CEA. Panel a compares acquisition disagreement under early commitment after two of six arrivals, complete-input shared history, and CEA across three benchmarks. Zero marks same-order repeat variation. Panel b separates input-token cost and answer accuracy into aligned columns with common method rows. CEA maintains separate notes for each subtask; the control shares notes across all subtasks. CEA uses 6,877 mean input tokens versus 11,505 for shared notes, with accuracies of 87.3 and 86.0 percent. Both token and accuracy estimates have 95 percent confidence intervals. Panel c shows the paired time saved by local execution against the identical CEA program delayed until all evidence arrives. Positive values mean earlier completion. The first subtask decision finishes about eight seconds earlier, while the final response saves between about one and two seconds. Gray diamonds denote DeepSeek V4 and blue circles GPT-5.6.}
\label{fig:method-results}
\end{figure*}

\section{Stable Acquisition with Asynchronous Progress}
\label{sec:method-evaluation}

CEA aims to retain the stability and answer quality of complete-input reasoning while allowing independent work to proceed. We test stability and search allocation, then evaluate question-only plans. Evidence recovery, matched reasoning controls, and timed execution assess useful search and efficiency (Figure~\ref{fig:method-results}).

\subsection{Stability from Retrieval to Answers}
We first compare CEA with early commitment on controlled bridge comparisons. Excess acquisition disagreement falls from 21.5 to $-0.2$\,pp with DeepSeek V4 and from 24.4 to $-0.7$\,pp in the matched GPT evaluation. Changing delivery order now produces about as much variation as repeating the same order. The corresponding answer contrasts are $-0.6$ and $-0.3$\,pp. On dependent chains, primary answer disagreement also falls from 12.1 to 1.0\,pp (Appendix Table~\ref{tab:answer-stability}).

Stable queries and stable acquisitions are distinct requirements. A branch may form the same request under every schedule, yet lose the opportunity to execute it when another branch finishes first.
To test why stable query formation needs stable allocation, we keep local reasoning fixed and give the shared search slot to the first completed request (Table~\ref{tab:allocation-controls}). With one supporting document withheld, only about 1\% of trajectories have two requesters, so fixed priority and first-ready rarely choose between competing queries.

\begin{table}[!ht]
\centering\small
\setlength{\tabcolsep}{3pt}
\newlength{\allocationnumwidth}
\settowidth{\allocationnumwidth}{$-00.0$}
\newcommand{\allocationnum}[1]{\makebox[\allocationnumwidth][r]{$#1$}}
\begin{tabular}{lccr@{\hspace{0.35em}}l}
\toprule
Allocation & Acq. (pp) & Acc. (\%) & \multicolumn{2}{c}{$\Delta$ Acq. (pp)} \\
\midrule
\multicolumn{5}{l}{\textbf{(a) One missing support: low contention}} \\
Fixed priority & \allocationnum{0.6} & \allocationnum{86.7} & \multicolumn{2}{c}{\textemdash} \\
First-ready & \allocationnum{0.7} & \allocationnum{87.3} & $+0.1$ & $[-2.1,\,2.4]$ \\
\addlinespace
\multicolumn{5}{l}{\textbf{(b) Two missing supports: high contention}} \\
Fixed priority & \allocationnum{-0.5} & \allocationnum{74.5} & \multicolumn{2}{c}{\textemdash} \\
First-ready & \allocationnum{21.1} & \allocationnum{74.1} & $+21.6$ & $[18.3,\,24.8]$ \\
\bottomrule
\end{tabular}
\caption{CEA allocation ablation on 100 2Wiki bridge-comparison questions (GPT-5.6; one search slot). Acq. denotes acquisition contrast and Acc. denotes accuracy. $\Delta$ is first-ready minus fixed [paired 95\% CI].}
\label{tab:allocation-controls}
\end{table}

With terminal support withheld from each branch on the same 100 questions, contention reaches 80.9--82.5\%. Local query-target disagreement remains near repeat variation under both rules, yet first-ready adds 21.6\,pp [18.3, 24.8] of acquisition disagreement, with similar accuracy.
The winning request determines which new evidence enters the selected branch and, through its answer, the final synthesis.

Each local proposal is formed from complete evidence, so the difference arises when competing requests are converted into an executed search. The shared budget couples otherwise independent branches at this step. A completion-order rule turns their relative completion times into a choice of evidence. CEA assigns this choice to a schedule-independent rule, allowing proposal formation to remain concurrent while keeping timing from selecting the evidence used by later reasoning (Figure~\ref{fig:domains}).

\subsection{Acquisition from Model-Generated Plans}
\label{sec:automatic-plans}
The controlled analysis uses annotated chains to define dependencies. Here, a planner sees only the question and specifies subquestions and initial retrieval queries. Each retrieved document belongs to the subquestion whose call returned it, giving CEA its dependency domains (Section~\ref{sec:setup}; Appendix~\ref{app:dependency}). Generated plans contain either a single reasoning chain or two independent subquestions. In the latter case, one branch can reason while the other awaits its inputs. We fix each generated plan and retrieved packet across policies to compare execution of the same task.

\begin{table}[!ht]
\centering\small
\setlength{\tabcolsep}{3pt}
\begin{tabular}{lcccr@{\hspace{0.35em}}l}
\toprule
Task & Early & Shared & CEA & \multicolumn{2}{c}{$\Delta$ vs. Early (pp)} \\
\midrule
2Wiki & $38.5$ & $47.9$ & $51.1$ & $+12.6$ & $[6.4,\,18.9]$ \\
HotpotQA & $38.9$ & $42.0$ & $42.5$ & $+3.6$ & $[0.3,\,7.1]$ \\
MuSiQue & $26.8$ & $25.8$ & $26.8$ & $-0.1$ & $[-3.6,\,3.5]$ \\
\bottomrule
\end{tabular}
\caption{Accuracy (\%) with question-only plans ($n=100$ per task). Early commits after two arrivals; Shared uses complete-input shared history. $\Delta$ is CEA minus Early [paired 95\% CI].}
\label{tab:automatic-results}
\end{table}

CEA retains complete-input stability and answer quality with these generated dependencies. Excess acquisition disagreement is 0.6, $-0.2$, and 1.0\,pp on 2Wiki, HotpotQA, and MuSiQue, comparable to the shared-history control that waits for all inputs (Figure~\ref{fig:method-results}a; Appendix~\ref{app:readings}). Accuracy is also comparable to that control, while exceeding early commitment by 12.6\,pp on 2Wiki and 3.6\,pp on HotpotQA, with a $-0.1$\,pp change on MuSiQue (Table~\ref{tab:automatic-results}; Appendix Table~\ref{tab:expansion-details}). CEA obtains its coordination structure from the agent's own retrieval calls, extending stable acquisition beyond annotated chains.

\subsection{Useful Evidence and Efficient Reasoning}
On the controlled two-domain task, withheld-support recovery measures whether stable acquisition remains useful (Section~\ref{sec:cost}). CEA raises recovery from 79.6\% to 86.4\% with DeepSeek~V4 and from 63.0\% to 81.4\% in the matched GPT evaluation, relative to complete-input \texttt{native@6}. The gain comes mainly from more frequent productive searches. CEA's corresponding accuracies are 90.5\% versus 89.8\% and 87.7\% versus 76.4\%, respectively (Appendix Tables~\ref{tab:cost-full} and~\ref{tab:supplement-details}).

We next isolate the contribution of private histories. A matched 100-question GPT control makes six note-taking calls over canonically ordered evidence for the full question and passes the shared history to both subquestion decisions. CEA retains only each domain's evidence and notes. The reasoning-call structure, allocation, feedback, and synthesis are matched. Both use 9.9 calls per trajectory, but CEA uses 6,877 input tokens versus 11,505, a 40.2\% reduction, with accuracies of 87.3\% and 86.0\% (Figure~\ref{fig:method-results}b). Keeping histories local avoids repeatedly carrying other subquestions' evidence and notes into later calls. By combining subquestion outputs only at final synthesis, CEA uses the same dependency boundaries to stabilize reasoning and limit the context carried through local calls.

CEA can also coordinate different local reasoning programs. Incremental note-taking raises primary accuracy from 67.2\% with Batch to 90.5\%; GPT Batch reaches 89.5\% versus 87.3\% with incremental reasoning (Appendix Table~\ref{tab:supplement-details}). GPT acquisition contrasts remain near repeat variation with CEA, CEA (Batch), and shared history. CEA's coordination rules accommodate either program, allowing local reasoning to be chosen for answer quality while preserving stable acquisition.

\subsection{Progress through Local Completion}
We time the same CEA program under two release rules. Local completion starts each domain's reasoning once all its inputs arrive; the global barrier waits for all initial evidence. On the two-domain task, first subquestion decisions finish 8.7\,s earlier with GPT (100 questions) and 8.3\,s earlier with DeepSeek V4 (20 questions). Across 800 matched GPT pairs, final-response time falls from 27.3 to 26.5\,s, a paired 3.1\% reduction [2.4, 3.8]. Accuracy is 88.3\% with local completion and 88.1\% with the barrier. DeepSeek V4 saves 3.9\% [1.9, 5.9] in response time (Figure~\ref{fig:method-results}c).

Local reasoning can finish during delivery, but final synthesis waits for all branch outputs and selected retrieval feedback (Section~\ref{sec:guarantee}). This explains why local decisions advance more than final responses. All three staggered arrival profiles show positive mean savings in both agents, while simultaneous delivery gives essentially none (Appendix Table~\ref{tab:supplement-details}). This pattern ties the speedup to overlap between reasoning and delivery. Appendix Table~\ref{tab:expansion-details} reports when the shared search budget permits retrieval to begin.

\section{Related Work}
\label{sec:related}

\paragraph{Evidence acquisition in search agents.}
Interactive search treats queries as evolving with acquired information~\cite{bates1989berrypicking}. Language-model search agents build on neural retrieval~\cite{karpukhin2020dpr} and generation conditioned on retrieved evidence~\cite{lewis2020rag,izacard2021fid}. Adaptive retrieval~\cite{jiang2023flare,asai2024selfrag,jeong2024adaptive}, browsing agents~\cite{nakano2021webgpt,li2025searcho1}, and learned tool-use policies~\cite{schick2023toolformer,jin2025searchr1,chen2025research} decide when and what to retrieve during reasoning. Decomposition and interleaved reasoning guide subsequent queries~\cite{press2023selfask,yao2023react,trivedi2023ircot}. GAIA, AgentBench, and BrowseComp evaluate general-assistant tasks, interactive environments, and difficult Web information seeking~\cite{mialon2024gaia,liu2024agentbench,wei2025browsecomp}. We complement these evaluations by examining evidence-path stability across delivery schedules.

\paragraph{Asynchronous execution and timing sensitivity.}
ReWOO separates planning, tool execution, and synthesis through evidence dependencies~\cite{xu2023rewoo}. LLMCompiler dispatches dependency-ready tasks concurrently~\cite{kim2024llmcompiler}, and RAG-R1 learns multi-query parallelism~\cite{tan2025ragr1}. AsyncLM incorporates tool returns into ongoing generation~\cite{gim2024asynclm}, while Speculative Interaction Agents reason during outstanding calls and revise pending work~\cite{hooper2026speculative}. Demonstration order~\cite{lu2022fantastically}, evidence position~\cite{liu2024lost}, and irrelevant context~\cite{shi2023distracted} influence reasoning. Multi-turn studies identify retained assumptions~\cite{laban2025multiturn} and divergent information-seeking paths~\cite{huang2026infoseek}. Concurrent work on ChronosAttack studies timing attacks on authentic tool responses~\cite{vashagh2026chronosattack}. Its causal barrier buffers a decision's observations for canonical presentation, sharing the complete-input waiting principle of our global-barrier control. We trace how timing affects subsequent corpus access. CEA coordinates local evidence dependencies and shared search allocation while allowing independent branches to progress.

\paragraph{Confluence and dependency coordination.}
Concurrency theory characterizes when execution order can vary without changing results. Kahn networks establish determinism under a fixed communication discipline~\cite{kahn1974}; conflict-free replicated data types converge through structured updates and merges~\cite{shapiro2011crdt}. CALM links coordination-free consistency to monotonicity~\cite{alvaro2011bloom,hellerstein2020calm}, while invariant confluence characterizes coordination requirements for application invariants~\cite{bailis2014coordination}. Recent work analyzes shared-state anomalies and concurrent external effects in multi-agent systems~\cite{khan2026concurrency}. CEA addresses stochastic search, where retrieval supplies later reasoning inputs and locally stable proposals can still compete for a shared budget. By coordinating local evidence and shared search, CEA preserves the joint distribution of decisions, acquisitions, and answers.

\section{Conclusion}
\label{sec:discussion}
\label{sec:conclusion}

Delivery timing shapes which evidence search agents acquire from identical initial inputs, with differences that persist through input completion and a second search. Timing influences both the evidence available when queries are formed and which proposed query receives a shared search slot. Acquisition stability captures the reproducibility of the evidence behind an answer, complementing task accuracy.

Our acquisition-regret bound explains why invariance at a fixed early arrival count can restrict the use of informative evidence already received. CEA instead coordinates local evidence completion and shared search allocation. For a fixed plan and stable conditional response law, it guarantees joint invariance of decisions, acquisitions, and answers while allowing independent work to proceed. With question-only plans, CEA retains acquisition stability and answer quality comparable to complete-input controls. Matched evaluations show lower context cost with private histories and earlier decisions and responses with local completion. Coordinating each decision's evidence and resource dependencies supports useful, reproducible acquisition without a global barrier while independent branches progress.

\par

\raggedbottom
\bibliographystyle{ACM-Reference-Format}
\bibliography{refs}
\appendix
\setlength{\emergencystretch}{2em}
\renewcommand{\thetable}{\thesection\arabic{table}}
\renewcommand{\theHtable}{app.\thesection.\arabic{table}}

\begingroup
\begin{table*}[!tp]
\centering\small
\renewcommand{\arraystretch}{0.95}
\setlength{\tabcolsep}{1.5pt}
\begin{minipage}[t]{0.49\textwidth}
\vspace{0pt}
\begin{tabular*}{\linewidth}{@{\hspace{1.5pt}\extracolsep{\fill}}lrrrrrr@{\hspace{1.5pt}}}
\toprule
Strategy & Search & Rec./S & Rec./all & Acc. & Strict & No alias \\
\midrule
\multicolumn{7}{l}{\textbf{DeepSeek V4 / Bridge comparison}\quad $n=200$} \\
\texttt{native@2} & $93.91$ & $14.04$ & $13.19$ & $59.53$ & $49.09$ & $59.53$ \\
\texttt{native@6} & $83.84$ & $94.93$ & $79.59$ & $89.84$ & $77.53$ & $89.84$ \\
\texttt{wait\_all@2} & $63.78$ & $5.44$ & $3.47$ & $68.97$ & \textemdash & \textemdash \\
\texttt{wait\_all@6} & $0.03$ & $100.0$\rlap{$^{\dagger}$}\phantom{0} & $0.03$ & $66.44$ & \textemdash & \textemdash \\
CEA (Batch) & $42.16$ & $90.96$ & $38.34$ & $67.22$ & $57.91$ & $67.22$ \\
CEA & $93.19$ & $92.69$ & $86.38$ & $90.53$ & $78.53$ & $90.53$ \\
\addlinespace[2pt]
\multicolumn{7}{l}{\textbf{DeepSeek V4 / Compositional}\quad $n=200$} \\
\texttt{native@2} & $83.72$ & $37.33$ & $31.25$ & $21.47$ & $21.06$ & $15.41$ \\
\texttt{native@6} & $70.66$ & $92.13$ & $65.09$ & $38.84$ & $38.31$ & $27.09$ \\
\texttt{wait\_all@2} & $70.03$ & $25.26$ & $17.69$ & $15.72$ & \textemdash & \textemdash \\
\texttt{wait\_all@6} & $17.59$ & $72.47$ & $12.75$ & $11.31$ & \textemdash & \textemdash \\
CEA & $70.50$ & $92.07$ & $64.91$ & $38.75$ & $38.22$ & $26.69$ \\
\addlinespace[2pt]
\multicolumn{7}{l}{\textbf{GPT-5.6 / Compositional}\quad $n=100$} \\
\texttt{native@2} & $89.31$ & $37.93$ & $33.88$ & $22.38$ & \textemdash & \textemdash \\
\texttt{native@6} & $79.75$ & $91.54$ & $73.00$ & $40.56$ & \textemdash & \textemdash \\
\texttt{wait\_all@2} & $90.25$ & $40.72$ & $36.75$ & $28.19$ & \textemdash & \textemdash \\
\texttt{wait\_all@6} & $84.69$ & $90.33$ & $76.50$ & $44.81$ & \textemdash & \textemdash \\
CEA & $78.81$ & $90.64$ & $71.44$ & $40.56$ & \textemdash & \textemdash \\
\addlinespace[2pt]
\end{tabular*}
\smallskip
\begin{tabular*}{\linewidth}{@{\hspace{1.5pt}\extracolsep{\fill}}lr@{\extracolsep{0pt}\hspace{0.35em}}l@{\hspace{1.5pt}}}
\toprule
\multicolumn{3}{l}{\textbf{Matched GPT bridge: answer contrasts} ($n=100$)} \\
\texttt{native@2} & $+6.6$ & $[3.3,\,9.9]$\,pp \\
\texttt{native@6} & $+0.3$ & $[-2.0,\,2.8]$\,pp \\
CEA & $-0.3$ & $[-1.2,\,0.8]$\,pp \\
\bottomrule
\end{tabular*}
\end{minipage}\hfill
\begin{minipage}[t]{0.49\textwidth}
\vspace{0pt}
\begin{tabular*}{\linewidth}{@{\hspace{1.5pt}\extracolsep{\fill}}lrrrrrr@{\hspace{1.5pt}}}
\toprule
Strategy & Search & Rec./S & Rec./all & Acc. & Strict & No alias \\
\midrule
\multicolumn{7}{l}{\textbf{DeepSeek V4 / HotpotQA}\quad $n=100$} \\
\texttt{native@2} & $74.00$ & $58.70$ & $43.44$ & $34.50$ & $31.19$ & $34.50$ \\
\texttt{native@6} & $48.88$ & $93.73$ & $45.81$ & $38.12$ & $35.62$ & $38.12$ \\
\texttt{wait\_all@2} & $57.63$ & $47.61$ & $27.44$ & $29.19$ & \textemdash & \textemdash \\
\texttt{wait\_all@6} & $8.06$ & $82.95$ & $6.69$ & $20.25$ & \textemdash & \textemdash \\
\addlinespace[2pt]
\multicolumn{7}{l}{\textbf{DeepSeek V4 / MuSiQue}\quad $n=100$} \\
\texttt{native@2} & $73.62$ & $19.02$ & $14.00$ & $27.12$ & $27.12$ & $26.44$ \\
\texttt{native@6} & $49.44$ & $34.01$ & $16.81$ & $26.06$ & $26.06$ & $25.25$ \\
\texttt{wait\_all@2} & $71.19$ & $14.22$ & $10.12$ & $18.81$ & \textemdash & \textemdash \\
\texttt{wait\_all@6} & $12.12$ & $20.62$ & $2.50$ & $13.31$ & \textemdash & \textemdash \\
\addlinespace[2pt]
\multicolumn{7}{l}{\textbf{GPT-5.6 / Bridge comparison}\quad $n=100$} \\
\texttt{native@2} & $84.81$ & $23.73$ & $20.12$ & $41.25$ & \textemdash & \textemdash \\
\texttt{native@6} & $70.00$ & $88.66$ & $62.06$ & $77.31$ & \textemdash & \textemdash \\
\texttt{wait\_all@2} & $76.06$ & $19.47$ & $14.81$ & $61.25$ & \textemdash & \textemdash \\
\texttt{wait\_all@6} & $46.62$ & $87.13$ & $40.62$ & $74.06$ & \textemdash & \textemdash \\
\addlinespace[2pt]
\end{tabular*}
\smallskip
\begin{tabular*}{\linewidth}{@{\hspace{1.5pt}\extracolsep{\fill}}lrrr@{\extracolsep{0pt}\hspace{0.35em}}l@{\hspace{1.5pt}}}
\toprule
\multicolumn{5}{l}{\textbf{Answer disagreement with DeepSeek V4} ($n=200$)} \\
Strategy & Same (\%) & Cross (\%) & \multicolumn{2}{c}{Centered (pp)} \\
\midrule
\multicolumn{5}{l}{\textbf{Bridge comparison: two domains}} \\
\texttt{native@2} & $32.62$ & $35.17$ & $+2.54$ & $[0.68,\,4.45]$ \\
\texttt{native@6} & $9.56$ & $9.50$ & $-0.06$ & $[-0.78,\,0.68]$ \\
CEA & $9.69$ & $9.11$ & $-0.58$ & $[-1.53,\,0.35]$ \\
\multicolumn{5}{l}{\textbf{Compositional: one domain}} \\
\texttt{native@2} & $14.44$ & $26.58$ & $+12.14$ & $[9.82,\,14.54]$ \\
\texttt{native@6} & $25.19$ & $26.48$ & $+1.29$ & $[-0.21,\,2.80]$ \\
CEA & $23.81$ & $24.77$ & $+0.96$ & $[-0.30,\,2.29]$ \\
\bottomrule
\end{tabular*}
\end{minipage}
\renewcommand{\thetable}{E1}
\caption{Strategy outcomes (\%) and answer disagreement. Rec./S and Rec./all report recovery per search and over all runs; $^{\dagger}$one search. Strict removes normalization; No alias removes aliases. Centered is cross-order minus same-order disagreement [95\% CI], treating NULL as a category. Dashes mark unreported values. Matched GPT contrasts use the replication in Table~\ref{tab:supplement-details}.}
\label{tab:cost-full}
\label{tab:scores}
\label{tab:answer-stability}
\end{table*}

\renewcommand{\thetable}{E2}
\begin{table*}[!tp]
\centering\small
\setlength{\tabcolsep}{3pt}
\newlength{\supplementnumwidth}
\settowidth{\supplementnumwidth}{$-00.0$}
\newcommand{\supplementnum}[1]{\makebox[\supplementnumwidth][r]{$#1$}}
\begin{tabular*}{\textwidth}{@{\hspace{3pt}\extracolsep{\fill}}lccr@{\extracolsep{0pt}\hspace{0.35em}}l@{\extracolsep{\fill}\hspace{6pt}}ccr@{\extracolsep{0pt}\hspace{0.35em}}l@{\hspace{3pt}}}
\toprule
& \multicolumn{4}{c}{DeepSeek V4 ($n=20$)} & \multicolumn{4}{c}{GPT-5.6 ($n=100$)} \\
\cmidrule(lr){2-5}\cmidrule(lr){6-9}
Arrival profile & Local (s) & Barrier (s) & \multicolumn{2}{c}{Saved (s)} & Local (s) & Barrier (s) & \multicolumn{2}{c}{Saved (s)} \\
\midrule
Simultaneous & \supplementnum{22.7} & \supplementnum{22.7} & $-0.1$ & $[-1.4,\,1.4]$ & \supplementnum{12.4} & \supplementnum{12.4} & $\phantom{+}0.0$ & $[-0.2,\,0.3]$ \\
Priority domain first & \supplementnum{40.6} & \supplementnum{42.0} & $+1.5$ & $[-0.5,\,3.4]$ & \supplementnum{30.9} & \supplementnum{32.3} & $+1.4$ & $[0.9,\,1.9]$ \\
Priority domain last & \supplementnum{40.3} & \supplementnum{42.5} & $+2.2$ & $[0.7,\,3.6]$ & \supplementnum{31.5} & \supplementnum{32.3} & $+0.8$ & $[0.5,\,1.2]$ \\
Fragmented & \supplementnum{40.2} & \supplementnum{42.5} & $+2.3$ & $[1.4,\,3.2]$ & \supplementnum{31.1} & \supplementnum{32.3} & $+1.2$ & $[0.8,\,1.6]$ \\
\addlinespace[2pt]
All profiles & \supplementnum{36.0} & \supplementnum{37.4} & $+1.5$ & $[0.7,\,2.2]$ & \supplementnum{26.5} & \supplementnum{27.3} & $+0.8$ & $[0.6,\,1.0]$ \\
\end{tabular*}
\smallskip
\begin{minipage}[t]{0.47\textwidth}
\setlength{\tabcolsep}{2pt}
\begin{tabular*}{\linewidth}{@{\hspace{2pt}\extracolsep{\fill}}lccc@{\hspace{2pt}}}
\toprule
\multicolumn{4}{l}{\textbf{Matched GPT replication ($n=100$)}} \\
Policy & Acq. (pp) & Acc. (\%) & Recovery (\%) \\
\midrule
\texttt{native@2} & \supplementnum{24.4} & \supplementnum{41.0} & $20.1$ \\
\texttt{native@6} & \supplementnum{1.4} & \supplementnum{76.4} & $63.0$ \\
CEA & \supplementnum{-0.7} & \supplementnum{87.7} & $81.4$ \\
\bottomrule
\end{tabular*}
CEA $-$ \texttt{native@6} accuracy: 11.3 [6.7, 16.0]\,pp.
\end{minipage}\hfill
\begin{minipage}[t]{0.51\textwidth}
\setlength{\tabcolsep}{2pt}
\begin{tabular*}{\linewidth}{@{\hspace{2pt}\extracolsep{\fill}}lcccr@{\hspace{2pt}}}
\toprule
\multicolumn{5}{l}{\textbf{GPT reasoning controls ($n=100$)}} \\
Policy & Acq. (pp) & Acc. (\%) & Calls & Input tokens \\
\midrule
CEA & \supplementnum{0.5} & \supplementnum{87.3} & $9.9$ & $6877$ \\
CEA (Batch) & \supplementnum{0.2} & \supplementnum{89.5} & $3.7$ & $2535$ \\
Shared history & \supplementnum{-0.2} & \supplementnum{86.0} & $9.9$ & $11505$ \\
\bottomrule
\end{tabular*}
CEA $-$ Shared accuracy: 1.3 [$-1.8$, 4.6]\,pp.\\
CEA $-$ Batch accuracy: $-2.3$ [$-4.7$, 0.2]\,pp.
\end{minipage}
\caption{Same-program timing and matched GPT-5.6 controls. Saved time is barrier minus local [paired 95\% CI]. Recovery is over all runs; calls and input tokens are means. Lower panels are separate matched comparisons.}
\label{tab:supplement-details}
\end{table*}

\begin{table*}[!t]
\begin{minipage}[t]{0.48\textwidth}
\vspace{0pt}
\renewcommand{\thefigure}{E1}
\renewcommand{\theHfigure}{app.E.1}
\centering
\includegraphics[width=\linewidth]{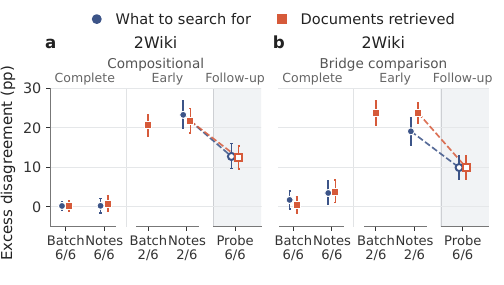}
\captionof{figure}{GPT-5.6 replication of Figure~\ref{fig:ladder}. Conditions, symbols, and scales are identical. Early decisions and their completed-input probes retain excess disagreement.}
\Description{Two GPT-5.6 panels show 2Wiki compositional and bridge-comparison questions. Conditions and metrics match Figure 2. Complete-input results are close to same-order repeat variation; early decisions and their follow-up probes show excess disagreement. Both Batch controls on bridge comparison are measured on the same fixed 100-question cohort.}
\label{fig:ladder-replication}

\end{minipage}\hfill
\begin{minipage}[t]{0.50\textwidth}
\vspace{0pt}
\renewcommand{\thetable}{E3}
\centering\small
\setlength{\tabcolsep}{2pt}
\newlength{\expansionnumwidth}
\settowidth{\expansionnumwidth}{$-00.0$}
\newcommand{\expansionnum}[1]{\makebox[\expansionnumwidth][r]{$#1$}}
\begin{tabular*}{\linewidth}{@{\hspace{2pt}\extracolsep{\fill}}lcccc@{\hspace{2pt}}}
\toprule
\multicolumn{5}{l}{\textbf{Fixed-packet acquisition contrasts (pp)}} \\
Task & Native@2 & Native@6 & Batch@2 & Batch@6 \\
\midrule
HotpotQA & $12.7$ & $3.3$ & $12.5$ & $1.2$ \\
MuSiQue & $11.7$ & $1.3$ & $11.2$ & $0.1$ \\
\end{tabular*}
\smallskip
\begin{tabular*}{\linewidth}{@{\hspace{2pt}\extracolsep{\fill}}lcr@{\extracolsep{0pt}\hspace{0.35em}}l@{\extracolsep{\fill}\hspace{4pt}}r@{\extracolsep{0pt}\hspace{0.35em}}l@{\hspace{2pt}}}
\toprule
\multicolumn{6}{l}{\textbf{Automatic plans: complete-input comparison (pp)}} \\
Task & Shared acq. & \multicolumn{2}{c}{$\Delta$ acq.} & \multicolumn{2}{c}{$\Delta$ acc.} \\
\midrule
2Wiki & \expansionnum{-1.3} & $+1.9$ & $[-1.1,\,5.0]$ & $+3.2$ & $[-1.0,\,7.6]$ \\
HotpotQA & \expansionnum{0.0} & $-0.2$ & $[-2.2,\,1.8]$ & $+0.5$ & $[-0.9,\,1.9]$ \\
MuSiQue & \expansionnum{1.3} & $-0.4$ & $[-3.0,\,2.2]$ & $+1.0$ & $[-1.0,\,3.0]$ \\
\bottomrule
\end{tabular*}
\caption{GPT-5.6 acquisition and accuracy comparisons ($n=100$ per task). Acq. denotes acquisition contrast, Acc. accuracy, and Batch \texttt{wait\_all}. $\Delta$ is CEA minus Shared [paired 95\% CI].}
\label{tab:expansion-details}

\smallskip
{\raggedright\noindent\textit{Paired fixed-packet contrasts (Native@2 $-$ Native@6).} 9.4 [5.9, 12.9]\,pp (HotpotQA) and 10.4 [6.7, 14.3]\,pp (MuSiQue).\par}

\smallskip
{\raggedright\noindent\textit{Search release (timed GPT).} CEA: 732 searches/800 runs, mean position 5.7; 118 (16.1\%) before completion. All 733 barrier searches start at position 6.\par}

\end{minipage}
\end{table*}
\endgroup
\section{Measurement Protocol}
\label{app:protocol}

\paragraph{Tasks and retrieval.}
Questions come from the 2Wiki and MuSiQue development sets and HotpotQA's distractor validation set. Corpus sizes are 398,354 (2Wiki), 509,315 (HotpotQA), and 117,534 (MuSiQue) paragraphs. Eligibility requires top-five retrieval of withheld support using its original title. The agent never receives this query. Non-supporting documents retrieved with the question fill the remaining slots. Duplicate excerpts, truncated support, and answer-exposing chains are excluded; distractors are not answer-filtered. SQLite FTS5 BM25 ($k_1=1.2$, $b=0.75$, equal title/body weights, ties broken by identifier) combines sorted unique Unicode alphanumeric tokens with OR, without stemming or stopword removal. A frozen tokenizer caps excerpts at 200 tokens; notes allow 256 output tokens and decisions/answers 512.

\paragraph{Models and probes.}
DeepSeek-V4-Flash uses temperature~0, top-$p$~1, and non-thinking mode. GPT-5.6 Luna uses low reasoning effort without temperature/top-$p$ settings. Unused early search opportunities expire, but delivery continues after answering. The completed-input probe uses the decision prompt without a budget notification. Target-eligible $n$ is 195, 200, 75 in Figure~\ref{fig:ladder}(a)--(c), and 98, 100 in Figure~\ref{fig:ladder-replication}(a)--(b).

\paragraph{Orders and estimation.}
We use eight orders (canonical, adjacent swaps, half-packet rotation, reversal), two repeats per order, and fixed identifiers and text. Equation~\ref{eq:centering} weights pairs equally within questions, then questions equally. Paired strategy intervals use 10,000 question-bootstrap resamples (seed root 20260914), retaining orders and repeats. Recovery per search uses cluster-ratio intervals. Missing acquisitions remain NULL; missing answers score zero.

\paragraph{Search-budget extensions.}
Each extension has 100 questions, eight orders, two repeats, and two policies (3,200 trajectories). The two-search extension retains both opportunities (Section~\ref{sec:setup}). The allocation experiment uses the fragmented profile below and varies only the allocation rule. First-ready selects the earliest completed search proposal. The contention extension retains the same bridge-comparison cohort and both roots, removes both terminal supports, and adds the question query's four highest-ranked non-supports.

\paragraph{Timed execution.}
Twenty DeepSeek and 100 GPT questions have two repeats per policy/profile. Profiles are simultaneous; priority-domain-first/last (canonical delivery over 0--4/10--20 seconds, singletons at endpoints); and fragmented (0, 4, 8, 12, 16, 20 seconds). Elapsed time spans dispatch to response; resampling retains profiles/repeats. Pseudorandom dispatch interleaves eight trajectories with at most sixteen concurrent requests.

\paragraph{Persistence.}
Table~\ref{tab:persistence} subtracts same-order from cross-order distance over eligible pairs, jointly resampling numerators and denominators. Demand requires valid actions; direction requires two search requests. NULL actions remain unknown and are separately checked as a category. Eligible pairs, not questions, receive equal weight.

\paragraph{Dependency-domain construction.}
\label{app:dependency}
Distractors are assigned by unique matches to branch entities in titles, then contents. Unassigned items fill the smaller domain in evidence-identifier order, with ties broken by branch identifier. Partitions are fixed across schedules.

\paragraph{Question-only plans.}
Planning precedes retrieval of six unique documents, three per domain in two-domain plans. Retrieval deduplicates the top-64 candidates, using the original question for fallback. Retrieval calls determine document ownership. Plans and packets remain fixed across policies and schedules. Two-domain counts are 38, 5, and 8 for 2Wiki, HotpotQA, and MuSiQue, respectively; other plans form one chain. Schema failures use single-chain fallbacks (29, 11, and 45, respectively). All questions and predefined strata are retained. Each question adds one planning call, averaging 171 input and 147 output tokens.

\paragraph{Trajectory coverage.}
\label{app:ledger}
The original 60,200 trajectories (59,200 formal, 1,000 reference) include 108 violations. Replication adds 4,800 complete runs; timing adds 480 (seven violations, 160 matched pairs). Budget extensions add 6,400 with full acquisition coverage; one missing first-ready answer scores zero. The 33,600 GPT expansion trajectories (12,800 cross-task, 4,800 reasoning, 1,600 timing, 14,400 automatic) include 33,291 complete runs, 307 violations, and two failures. Contention adds 3,200 complete runs; GPT bridge batch controls add 3,200 with three violations. Failures are retained; 156 structural pilots, three pilot plans, and 16 contention pilots are excluded.

\section{Proofs and Input Checks}
\label{app:input-audit}

\paragraph{Prefix characterization.}
Exact decision invariance equates laws along every edge and hence every path. Conversely, a packet's prefixes share a component. Thus invariance is equivalent to $\mu_S=\nu_{c(S)}$ for component $c(S)$. If all $n$-subsets are admissible, single-item replacements connect the graph ($m+1\le n$), forcing a constant law.

\paragraph{Acquisition utility and capacity.}
For query or answer actions $a$ and utility $u_E(a)\in[0,1]$, set $\ell_E(a)=\sup_{a'} u_E(a')-u_E(a)$. With support recovery as utility, each individually recoverable target has maximal utility one. Search capacity $b<k$ means $\sum_i u_{E_i}(a)\le b$ for every action. Hence $\Gamma\ge1-b/k$, giving average and maximum regret at least $[1-b/k-\rho_G\varepsilon]_+$ over the selected pairs.

\paragraph{Matching construction.}
Let the admissible packets be $E_i=B\cup U_i$ ($k\ge2$), with $|B|=m$, $|U_i|=n-m$, and $B,U_1,\ldots,U_k$ pairwise disjoint. Their prefix cliques share only $B$. Each selected $S_i\ne B$ identifies packet $i$ through an item from $U_i$. Distances are one to $B$ and two to prefixes exclusive to other packets, giving $\rho_G=1$.

Assign distinct targets to packets and one query to each $b$-subset ($1\le b<k$), recovering exactly that subset; other actions recover none. Then $\Gamma=1-b/k$. Let $\nu$ be uniform over queries, with $\nu_i^+,\nu_i^-$ conditioned on including/excluding target $i$. Set $p=b/k$, $\delta=\min\{\varepsilon,1-p\}$, and
\[
\mu_B=\nu,\qquad
\mu_V=(p+\delta)\nu_i^++(1-p-\delta)\nu_i^-
\]
for every prefix $V\ne B$ exclusive to packet $i$. Disjoint conditional supports give $\operatorname{TV}(\mu_V,\nu)=\delta\le\varepsilon$; other edges join identical laws. Each selected prefix recovers its target with probability $p+\delta$, so
\[
r_{E_i}(S_i)=1-p-\delta=[1-b/k-\varepsilon]_+.
\]
This attains Equation~\ref{eq:stability-regret} on informative prefixes for every $\varepsilon\ge0$, including $n=6,m=2$.

\paragraph{Contraction under scoring.}
Fix question, packet, settings, and nonce. Let $P_\pi$ be the joint acquisition/decision/answer law, $g$ select the answer, and $h$ score binary correctness. Accuracy is $a_\pi=\E_{P_\pi}[h\circ g]$, giving
\begin{equation}
\begin{split}
|a_\pi-a_\sigma|
&=\operatorname{TV}((h\circ g)_\#P_\pi,(h\circ g)_\#P_\sigma)\\
&\le\operatorname{TV}(g_\#P_\pi,g_\#P_\sigma)\\
&\le\operatorname{TV}(P_\pi,P_\sigma).
\end{split}
\label{eq:readout-contraction}
\end{equation}
Equality follows from Bernoulli total variation, and deterministic mappings give the inequalities. Distinct acquisitions can yield the same correct answer, with retrieval TV of one and answer TV of zero. Total variation here compares outcome distributions; Equation~\ref{eq:centering} measures excess pairwise disagreement.

\paragraph{Fixed-priority release.}
Under fixed priority, an available search proposal from domain $j$ is eligible exactly when every $i<j$ has declined. Pending predecessors can preempt $j$; successors cannot. Selecting no search requires all domains to decline. The winner becomes eligible at $\max_{i\le j^\star}\tau_i$ (Equation~\ref{eq:allocation-release}).

\paragraph{Input checks.}
Canonical JSON checks confirm identical batch inputs and incremental non-model messages across orders at fixed question, partition, repeat, domain, and call. All 300 planner inputs contain only questions and common instructions.

\section{Repeat Controls and Interpretation}
\label{app:reliability}
Repeats use distinct nonces shared across orders. Mixing same-repeat distance $a$ and different-repeat distance $c$ equally across orders, but using only $c$ within orders, gives contrast $(a-c)/2$ even for identical inputs. Equation~\ref{eq:centering} uses different-repeat pairs in both terms. For independent categorical draws with laws $p_{\pi r}$, let $\delta_r=p_{\pi r}-p_{\sigma r}$. Then
\begin{equation}
\E[\widehat\Delta_{\pi\sigma}]=\tfrac12\delta_0^{\mathsf T}\delta_1.
\end{equation}
Invariance for each nonce implies zero expected contrast, but the converse fails because nonce-specific effects can cancel. We retain negative contrasts.

\section{Answer Scoring}
\label{app:scoring}
Scoring case-folds Unicode, removes ASCII punctuation and whole-word \emph{a/an/the}, and normalizes whitespace. Empty or missing answers score zero. Exact matching uses answer-ID-linked 2Wiki aliases/demonyms, official MuSiQue aliases, and canonical HotpotQA answers, with string matching for numbers and dates. All 452 numerically parseable matches in the primary audit agreed in value.

Oracle/closed-book reference accuracies (\%, $n=100$) are 94/68 (bridge), 49/21 (compositional), 51/23 (HotpotQA), and 51/16 (MuSiQue) for DeepSeek; GPT compositional values are 45/31.

\section{Additional Results}
\label{app:readings}
\label{app:cost}
In the primary evaluation ($n=200$), early commitment yields excess disagreement in support recovery of 7.7 [5.9, 9.7]\,pp (bridge comparison) and 19.3 [16.9, 21.7]\,pp (compositional). With automatic plans, acquisition-contrast reductions (Early $-$ CEA) are 19.8 [16.5, 23.1], 14.3 [10.3, 18.3], and 7.9 [4.8, 11.3]\,pp on 2Wiki, HotpotQA, and MuSiQue, respectively.

In Table~\ref{tab:allocation-controls}a, fixed/first-ready recovery is 81.8/84.0\%, with 0.91 searches for both. In (b), contention is 80.9/82.5\%, either-target recovery 84.4/86.6\%, and per-domain recovery 42.2/43.3\%. First-ready $-$ fixed differences are 0.0 [$-2.6$, 2.6]\,pp in query-target contrast, 24.0 [20.8, 27.2]\,pp in selected-branch contrast, and $-0.4$ [$-4.0$, 3.4]\,pp in accuracy. Both-target recovery is 0/1,588 for fixed priority and 0/1,587 for first-ready searches.

\end{document}